\documentclass[letterpaper]{article} 
\usepackage{aaai25}  
\usepackage{times}  
\usepackage{helvet}  
\usepackage{courier}  
\usepackage[hyphens]{url}  
\usepackage{graphicx} 
\usepackage{natbib}  
\usepackage{caption} 
\usepackage{algorithmic}
\usepackage{threeparttable}
\usepackage{diagbox}
\usepackage{multirow}

\usepackage{latexsym, amsmath, amssymb, amsfonts, mathrsfs}
\usepackage{xspace}
\usepackage[T1]{fontenc}
\usepackage{float}
\usepackage{color}
\usepackage{url}
\usepackage[ruled,linesnumbered,noresetcount,vlined]{algorithm2e}
\usepackage{todonotes}
\usepackage{picinpar}
\usepackage{multirow}
\usepackage{subcaption}
\usepackage{amsthm}
\usepackage{booktabs} 
\usepackage{diagbox}
\allowdisplaybreaks
\usepackage{pgfplots}

 \newcommand{\B}{\mathcal{B}}

\newcommand{\fml}[1]{\mathcal{#1}}

\newcommand{\ind}{\perp\!\!\!\!\perp} 

\newcommand{\beitemize}{\begin{list}{$\bullet$}{\topsep=1.5pt \parsep=0pt \itemsep=1pt \leftmargin=1em }} 
\newcommand{\enitemize}{\end{list}}

\newcommand{\beenumerate}{\hspace{-0.5in} \begin{enumerate}\topsep=1pt \parsep=0pt \itemsep=-3pt} \newcommand{\enenumerate}{\end{enumerate}}

\newcommand{\belist}{\begin{list}{$\bullet$}{\topsep=1.5pt \parsep=0.5pt \itemsep=1pt \leftmargin=2.25em \labelwidth=1.0em \labelsep=0.5em \partopsep=1.5pt}} 
\newcommand{\enlist}{\end{list}}

\newtheorem{theorem}{{\bf Theorem}}

\newtheorem{definition}{{\bf Definition}}

\newtheorem{example}{{\bf Example}}
\newtheorem{proposition}{{\bf Proposition}}

\newboolean{includeMemo}
\setboolean{includeMemo}{true} 

\newcommand{\memoside}[1]{\ifthenelse{\boolean{includeMemo}}{\todo[caption={},color=green!20!]{{\footnotesize #1}}}}
\newcommand{\memo}[1]{\ifthenelse{\boolean{includeMemo}}{\todo[inline,caption={},color=green!20!]{#1}}}
\newcommand{\memob}[1]{\ifthenelse{\boolean{includeMemo}}{\todo[inline,caption={},color=blue!20!]{#1}}}

\newcommand{\xhdr}[1]{\vspace{5pt}\noindent\textbf{#1 }}
\newcommand{\ignore}[1]{}

\newcommand{\squishlist}{
\begin{list}{{{\small{$\bullet$}}}}
{\setlength{\itemsep}{3pt}      
\setlength{\parsep}{3pt}
\setlength{\topsep}{3pt}       
\setlength{\partopsep}{3pt}
\setlength{\leftmargin}{1em} 
\setlength{\labelwidth}{1em}
\setlength{\labelsep}{0.5em} } }
\newcommand{\squishend}{  \end{list}}

\newcommand{\squishenum}{
\begin{list}{$\bullet$}{ 
    \setlength{\itemsep}{1pt}
    \setlength{\parsep}{0pt}
    \setlength{\topsep}{1.5pt}
    \setlength{\partopsep}{0pt}
    \setlength{\leftmargin}{2em}
    \setlength{\labelwidth}{1.5em}
    \setlength{\labelsep}{0.5em} } }

\usepackage{newfloat}
\usepackage{listings}
\DeclareCaptionStyle{ruled}{labelfont=normalfont,labelsep=colon,strut=off} 
\floatstyle{ruled}
\newfloat{listing}{tb}{lst}{}
\floatname{listing}{Listing}
 \nocopyright 

\title{Human-Aware Belief Revision: A Cognitively Inspired Framework for Explanation-Guided Revision of Human Models
}

\author{
    Stylianos Loukas Vasileiou,
    William Yeoh
}
\affiliations{

    Washington University in St. Louis\\
    \{v.stylianos, wyeoh\}@wustl.edu
}

\begin{document}

\maketitle

\begin{abstract}
Traditional belief revision frameworks often rely on the principle of minimalism, which advocates minimal changes to existing beliefs. However, research in human cognition suggests that people are inherently driven to seek explanations for inconsistencies, thereby striving for explanatory understanding rather than minimal changes when revising beliefs. Traditional frameworks often fail to account for these cognitive patterns, relying instead on formal principles that may not reflect actual human reasoning. To address this gap, we introduce \textit{Human-Aware Belief Revision}, a cognitively-inspired framework for modeling human belief revision dynamics, where given a human model and an explanation for an explanandum, revises the model in a non-minimal way that aligns with human cognition. Finally, we conduct two human-subject studies to empirically evaluate our framework under real-world scenarios. Our findings support our hypotheses and provide insights into the strategies people employ when resolving inconsistencies, offering some guidance for developing more effective human-aware AI systems.

\end{abstract}

\section{Introduction}
When faced with new information that challenges their existing beliefs, people naturally seek to adjust their beliefs to accommodate this new information. Determining the most rational way to make such adjustments is not always straightforward. While rational decision-making helps people identify inconsistencies, it does not necessarily guide them on which beliefs to revise to restore consistency.

One perspective on rational change is the principle of \textit{minimalism} (or information economy), which emphasizes a minimal change to one's beliefs. \citet{james1907pragmatism} eloquently expressed this notion, suggesting that new information should be integrated in a way that slightly stretches people's existing beliefs just enough to incorporate the new information. On the one hand, minimalism is not only practical but has also been championed by numerous cognitive scientists in the context of both scientific and everyday decision-making \citep{gardenfors1988knowledge,harman1986change}. Nonetheless, minimalism can be limiting; it does not always yield the most effective or insightful belief revisions, particularly in complex real-world situations.

Cognitive studies on human belief revision indicate that people may not always aim for minimality when revising their beliefs. Instead, they first seek to understand the nature of the inconsistency \citep{thagard1989explanatory,johnson2004reasoning}. When encountering conflicting information, individuals often generate \emph{explanations} to reconcile these inconsistencies, as explanations offer a clearer guide for future actions than mere belief adjustments \citep{craiknature,keil2006explanation} and play a vital role in communicating one's understanding of the world \citep{chi1994eliciting,lombrozo2007simplicity}. This approach suggests that people construct explanations to resolve inconsistencies, leading to belief revisions that are often not minimal \citep{johnson2004reasoning,khemlani2013cognitive}.

By drawing inspiration from the aforementioned cognitive studies, we challenge the primacy of minimalism in belief revision frameworks aimed at human users. We argue that a drive for \textit{explanatory understanding}, rather than mere consistency of beliefs, is a key feature of human reasoning that belief revision frameworks should account for. Building on this foundation, we introduce the \emph{human-aware belief revision} framework, which formalizes the process of revising human beliefs in light of explanations. Central to our framework is the \textit{explanation-guided} revision operator that, given a human model and an explanation for an explanandum (the fact to be explained), revises the model in a (possibly) non-minimal way while preserving the new information.

Moreover, we present empirical findings from two human-subject studies that explore human belief revision behavior in real-world scenarios. These studies not only provide support for the explanation-guided revision but also yield valuable insights into the strategies people employ when creating explanations for resolving inconsistencies. The inclusion of empirical evaluation sets our work apart from purely theoretical approaches, grounding it in real-world human reasoning patterns.

In summary, our contributions include: (1) We introduce a framework for human-aware belief revision inspired by human cognition. We focus on explanatory understanding over minimal changes, and show how to construct the explanation-guided revision operator; (2) We conduct empirical evaluations through two human-subject studies that provide robust evidence for the applicability of our proposed framework. Our main thesis is that \textit{an explanation-guided belief revision operator aligns more closely with natural human reasoning patterns, thus offering a human-aware approach to belief revision.}

\section{Related Work}
\label{sec:related}

The study of belief revision has been a focal point of research in both philosophy and AI, evolving over several decades. Among the most influential frameworks in the this domain is the AGM framework \cite{alchourron1985logic,gardenfors1988knowledge,gardenfors1984epistemic}. While the AGM framework has been highly influential \cite{ferme2018belief}, it unfortunately does not allow for an account of explanation. Crucially, the hallmark of AGM is the principle of minimalism, a point that we critically reconsidered in this paper. When it comes to human belief revision, the explanatory understanding together with our empirical results indicate that humans tend to perform non-minimal revisions to their beliefs that are guided by explanations.

Our work is very closely related with the belief revision operator proposed by \citet{falappa2002explanations}, where they emphasized the role of explanations in belief revision, particularly when the incoming information is inconsistent with the existing beliefs. However, there are some subtle differences: their revision operator adheres to minimalism, and makes it possible for the explanation to be rejected, consequently rejecting the explanandum. In the appendix, we provide an example that better illustrates these differences.

We are not the first to question adherence to minimal changes in belief revision. Notably, \citet{rott2000two} has critically examined minimalism in the context of AGM theory. He argued that, while intuitively appealing, may not adequately capture the complexities of real-world belief revision processes. \citet{kern2002principle} has also addressed the inadequacy of the minimalist approach guided by the AGM postulates in preserving conditional beliefs (if-then statements) during revision, and presented a thorough axiomatization of conditional preservation in belief revision. Nonetheless, in our framework we do not make such restrictions, and importantly, the empirical results we obtained refute the preservation of conditionals, at least in human belief revision.

Cognitive scientists and psychologists have also critiqued the principle of minimal change in belief revision, showing that human reasoners use a different strategy when revising their beliefs with new, conflicting information \citep{elio1997belief,politzer2001belief,elio2022disbelieve,khemlani2013cognitive,johnson2004reasoning}, namely the \textit{explanatory understanding}. As we expressed throughout this paper, this approach suggests that people create explanations to resolve inconsistencies, leading them to make greater than minimal changes to the information that they have.

Finally, our work finds direct relevance in the domain of human-aware AI \citep{kambhampati2020challenges}. A key component of such systems is the assumption of a human (mental) model \cite{carroll1988mental}, which represents the AI agent's understanding of the human's knowledge, beliefs, goals, and decision-making processes. Generally, the human models are incorporated into the agent's deliberative processes in order to generate better explanations that align with the respective human users \citep{chakraborti2017plan,sreedharan2021foundations,son2021model,vas21,vasileioulogic}. Therefore, for effective human-AI interaction, these models must be maintained and updated frequently with new information.
Nevertheless, most existing works adhere to the principle of minimalism, that is, they apply minimal changes to the human models. This approach, while intuitive, may not always accurately reflect the nature of human belief revision. Our proposed framework in this paper is aimed at enhancing such works by providing insights into how people actually revise their beliefs when faced with inconsistencies.

\section{Background}
\label{sec:background}

\subsection{Belief Revision Theory}
\label{sec:revision}

Practical applications of belief revision theories typically assume that the agent's beliefs are represented as non-closed sets of formulae called \textit{belief bases} \citep{hansson1994kernel,hansson1999survey}.\footnote{Theoretical approaches to belief revision also consider sets of formulae that are closed under a consequence relation \citep{alchourron1985logic}. See \citet{ferme2018belief} for an overview of different approaches to belief revision.} There, revision operators are typically constructed via \textit{kernel functions} \citep{hansson1994kernel} that select among (minimal) subsets of a belief base that contribute to making it imply a formula. This is a general approach that relies on incision functions to determine the beliefs to be removed from each kernel.

\xhdr{Minimalism}
\label{sec:discussion:minimalism}
The \textit{principle of minimalism} (also called the principle of informational economy or minimal mutilation) is one of the basic conceptual principles underlying belief revision frameworks:

\begin{quote}
    \textbf{Principle of Minimalism:} When an agent with a prior belief base is presented with a new belief that is inconsistent with it, they should revise it with respect to the new belief to get a posterior belief base that is \textit{the closest belief base to their prior belief base}.
\end{quote}

In essence, the principle of minimalism states that the agent's primary goal when resolving inconsistencies is to make minimal changes to their (existing) beliefs. This basic principle is encapsulated by certain postulates, such as recovery \citep{gardenfors1982rules,alchourron1985logic}, core-retainment, and relevance \citep{hansson1991belief,falappa2002explanations}, or by constructing revision operators that restrict revisions to minimal subsets of the original belief base \cite{falappa2002explanations}.

\xhdr{Explanatory Understanding} However intuitive and plausible minimalism appears to be, it does not always hold true in human belief revision dynamics. Studies in cognitive science show that when people are faced with inconsistencies, they first construct (or seek) explanations to resolve the inconsistencies, which consequently lead them to revise their beliefs in a non-minimal fashion \citep{elio1997belief,johnson2004reasoning,khemlani2013cognitive,walsh2009changing}. We will refer to this as \textit{explanatory understanding}:

\begin{quote}
    \textbf{Explanatory Understanding:} When an agent with a prior belief base is presented with a new belief that is inconsistent with it, they first seek for an explanation to explain the origin of the inconsistency, which can then lead them to make greater than minimal changes to their belief base.
\end{quote}

The explanatory understanding poses a challenge to the principle of minimalism, and thus to most belief revision frameworks that adhere to it. Crucially, this means that existing belief revision frameworks may fail to be successfully applied to human-AI settings. In this work, we argue against minimalism and for explanatory understanding in human-centric belief revision frameworks.

\subsection{Logical Preliminaries}

We will adopt an equality-free first-order language $\mathcal{L}$ comprising (finite) sets of constants $\mathcal{C}$, variables $\mathcal{V}$, predicates $\mathcal{P}$, and no explicit existential quantifiers. An atom takes the form $p(t_1, \ldots,t_n)$, where $p$ is a predicate and $t_i \in \mathcal{C} \cup \mathcal{V}$ are terms. A \textit{ground} (or instantiated) atom is an atom without variables, otherwise it is \textit{lifted}. A formula is built out of atoms using quantifier $\forall$ and the usual logical connectives $\neg$, $\vee$, and $\land$. An \textit{interpretation} is an assignment of $true$ or $false$ to each ground atom in a set of formulae. If an interpretation satisfies a set of formulae then it is called a \textit{model} of that set. A set of formulae is \textit{consistent} if it has at least one model, otherwise it is \textit{inconsistent}. 

To align with the cognitive flexibility observed in human reasoning, we make the fundamental assumption that beliefs are \textit{defeasible}---that is, they are open to change and retraction. To capture a more nuanced structure of human beliefs, we represent a human model as a \textit{belief base} consisting of \textit{facts} and \textit{general rules} from language $\mathcal{L}$. This approach allows us to model both specific pieces of information and the broader principles that guide human reasoning. Formally,

\begin{definition}[Belief Base]
\label{def:model}
A \textit{belief base} is a tuple $\B = \langle \mathcal{F}, \mathcal{R} \rangle$, where:
\begin{itemize}
    \item $\mathcal{F} \subseteq \mathcal{L}$ is a set of facts, each of which is a ground atom or its negation.
    \item $\mathcal{R} \subseteq \mathcal{L}$ is a set of general rules, each having the form of a conditional $P(\vec{x}) \rightarrow Q(\vec{y})$, where $\vec{x}$ and $\vec{y}$ are variable tuples, and $P(\vec{x})$ and $Q(\vec{y})$ are formulae in $\mathcal{L}$.
\end{itemize}
We denote the ground version of $\B$ as $\B^\gamma$.
\end{definition}

A belief base $\B$ entails a formula $\varphi$, denoted by $\B \models \varphi$, iff $\B^\gamma \cup \lnot \{\varphi\} \models \bot$, where $\bot$ denotes falsum. We will also use $\Gamma(\B)$ to denote the (primitive) \textit{consequences} (e.g., set of all ground atoms ) of $\B$, i.e., $\Gamma(\B) = \{ \phi \: | \: \phi \in \fml{L}, \B \models \phi \}$.\footnote{This is also called the backbone of the belief base \cite{parkes1997clustering}.}

For convenience, and unless specified otherwise, we will write a belief base as the set $\B = \fml{F}\cup \fml{R}$, implicitly assumed to be a conjunction of the facts and rules.

\section{Towards Human-Aware Belief Revision}
\label{framework}

Consider an illustrative scenario involving two agents, Alice and Bob. Suppose that Alice believes that (a) \textit{If people are worried, then they have insomnia}, (b) \textit{Diana is worried}, and (c) \textit{Charlie is worried}. As such, she believes that (d) \textit{Charlie and Diana have insomnia}. Now, suppose that Bob tells Alice that (e) \textit{Charlie does not have insomnia}. Before incorporating this new information that contradicts her beliefs, Alice will naturally seek an explanation from Bob. For instance, Bob would then explain to Alice that (f) \textit{Charlie has a coping strategy} and that (g) \textit{people with coping strategies may not have insomnia despite being worried}. If the explanation stands up to scrutiny, Alice will then integrate this explanation and into her beliefs. In other words, it is the explanation that will drive Alice's revision process.

An explanation aims to explain (or, ``rationalize'') a particular phenomenon, referred to as the \textit{explanandum}, to someone. Essentially, an explanation is reasoning in ``reverse'', e.g., a set of beliefs adduced as explanations of the explanandum. More formally,

\begin{definition}[Explanation]
\label{def:explanation}
We say that $\fml{E} \subseteq \mathcal{L}$ is an \emph{explanation} for an explanandum $\varphi \subseteq \mathcal{L}$ iff: (1) $\fml{E} \models \varphi$; (2) $\fml{E} \not \models \bot$; and (3) $\forall \fml{E}' \subset \fml{E}$, $\fml{E}' \not \models \varphi$.
\end{definition}

The first condition determines that the explanandum is derived by the set of beliefs $\fml{E}$, while the second condition averts the possibility that an explanandum is derived from an inconsistent set. The final condition ensures that there are no irrelevant beliefs in the explanation. Note that we assume the explanation is also a belief base (Definition~\ref{def:model}).

\begin{example}

   Building~on~the~illustrative scenario, Alice's belief base is $\B_A = \{\text{Wor}(Charlie), \text{Wor}(Diana), \forall x. \text{Wor}(x)\!\!\rightarrow\!\!\text{Ins}(x) \}$. It is easy to see that $\B_A \models \text{Ins}(Charlie)$. Upon encountering the contradictory explanandum $\lnot \text{Ins}(Charlie)$, the explanation for it from Bob is $\fml{E} =  \{\text{Wor}(Charlie), \text{Cop}(Charlie), \forall x. \text{Wor}(x) \land \text{Cop}(x) \rightarrow \lnot \text{Ins}(x) \}$.

\end{example}

What is important to note here is that the explanation should aim to enable Alice's \textit{explanatory understanding} of the explanandum, that is, Alice understands why Charlie does not have insomnia if she can produce an explanation for it. Therefore, a good explanation should drive the receiving agent's belief revision in such a way that explanatory understanding of the given explanandum is satisfied.

\begin{definition}[Explanatory Understanding]
\label{def:explanatory_understanding}
    Let $\B$ be the belief base of an agent, $\fml{E}$ an explanation for explanandum $\varphi$, and $\odot$ a revision operator such that $\B \odot \fml{E}$ denotes the revision of $\B$ with $\fml{E}$. We say that the agent has explanatory understanding of $\varphi$ after receiving $\fml{E}$ if $\B \odot \fml{E} \models \varphi$.
    \label{expl-u}
\end{definition}

\begin{example}
Continuing from Example~\ref{exp2}, notice that the Alice's belief base $\B_A$ and the explanation $\fml{E}$ are inconsistent, i.e., $\B_A \cup \fml{E} \models \bot$. Thus, a revision operator $\odot$ must re-establish consistency within Alice's belief base while enabling explanatory understanding of the explanandum $\lnot In(Charlie)$. But how this revision should happen? Should Alice, in light of $\fml{E}$, still maintain the general belief that people have insomnia if they are worried?
\end{example}

When resolving inconsistencies with explanations, we would like to have a revision operator that, on the one hand, yields a consistent revised belief base, but on the other hand guarantees explanatory understanding of the explanandum. But how should such a revision operator behave? In what follows, we argue it should not be constrained by the principle of minimalism, e.g., minimal changes to the belief base, but rather be guided by the explanation, which might lead to larger than minimal revisions.

\subsection{The Explanation-Guided Revision Operator}

We build upon the concept of \textit{kernel revision} \cite{hansson1994kernel}. Particularly, we assume that the belief base $\B$ we are dealing with is fully grounded,\footnote{This restriction is only made for convenience and illustrative purposes. The results continue to hold for lifted belief bases.} and start by defining the notion of \emph{correction kernel}:

\begin{definition}[Correction Kernel]
    Let $\B$ be a belief base and $\fml{E}$ an explanation for explanandum $\varphi$. If $\B \cup \fml{E} \models \bot$, then the correction kernel of $\B \cup \fml{E}$ is defined as $(\B \cup \fml{E})^{\bot} = \{\B' \: | \: \B' \subseteq \B \cup \fml{E}, (\B \cup \fml{E}) \setminus \B' \not \models \bot, \text{ and } (\B \cup \fml{E}) \setminus \B' \not = \emptyset \}$. If $\B \cup \fml{E} \not \models \bot$, then $(\B \cup \fml{E})^{\bot} = \B \cup \fml{E}$.

\end{definition}

In other words, the correction kernel of an inconsistent belief base is the set of all consistent subsets of the belief base whose removal render the belief base consistent. The elements of the correction kernel are called \emph{correction sets}. Note how we do not impose any minimality constraints on the resulting correction sets.

\begin{example}
    Consider the belief base $\B= \{\text{Wor}(Charlie) , \text{Wor}(Charlie) \rightarrow \text{Ins}(Charlie)\}$ and explanation $\fml{E} = \{\lnot \text{Ins}(Charlie) \}$ for explanandum $\lnot \text{Ins}(Charlier)$. The correction kernel of $\B \cup \fml{E}$ is $(\B \cup \fml{E})^\bot = \{ \{ Wo(Charlie)\}$, $\{ \text{Wor}(Charlie)\rightarrow \text{Ins}(Charlie)\}$, $\{ \lnot \text{Ins}(Charlie) \}$, $\{ \text{Wor}(Charlie), \text{Wor}(Charlie) \rightarrow \text{Ins}(Charlie)\}$, $\{ \text{Wor}(Charlie), \lnot \text{Ins}(Charlie)\}$, $\{ \lnot \text{Ins}(Charlie), \text{Wor}(Charlie) \rightarrow \text{Ins}(Charlie)\} \}$.
\end{example}

As mentioned earlier, we want our operator to not only restore consistency (with no minimality guarantees), but also to enable explanatory understanding of the explanandum $\varphi$ (Definition~\ref{expl-u}). This leads us to define the notion of the \emph{$\varphi$-preserving (explanandum-preserving) selection function}:

\begin{definition}[$\varphi$-Preserving Selection Function]
    Let $\B$ be a belief base and $\fml{E}$ an explanation for explanandum $\varphi$. A $\varphi$-preserving selection function for $\B$ with respect to $\fml{E}$ is defined as $\Sigma: 2^{2^{\fml{L}}} \mapsto 2^{\fml{L}}$ such that: (1)~$\Sigma((\B \cup \fml{E})^\bot) \in (\B \cup \fml{E})^\bot$; and (2) $ (\B \cup \fml{E}) \setminus \Sigma((\B\cup \fml{E})^\bot) \models \varphi$.
\end{definition}

A $\varphi$-preserving selection function selects a correction set from the correction kernel $(\B\cup \fml{E})^\bot$ (condition (1)) whose removal does not affect the entailment of the explanandum $\varphi$ (condition (2)). Note that in the case where $\B\cup \fml{E} \not \models \bot$, then $\Sigma((\B\cup \fml{E})^\bot) = \emptyset$.

\begin{example}
    Consider the belief base and explanation from Example~\ref{exp2}. Possible results of the $\lnot In(Charlie)$-preserving selection function $\Sigma((\B\cup \fml{E})^\bot)$ are $\{ \text{Wor}(Charlie)\}$, $\{  \text{Wor}(x) \rightarrow \text{Ins}(x)\}$ and $\{\text{Wor}(Charlie), \text{Wor}(x) \rightarrow \text{Ins}(x)\}$.
    \label{example3}
\end{example}

Having defined the correction kernel and the $\varphi$-preserving selection function, we now formally define the \textit{explanation-guided belief revision operator}, which is a function mapping a belief base and an explanation to a revised belief base:

\begin{definition}[Explanation-Guided Belief Revision Operator]
    Let $\B$ be a belief base, $\fml{E}$ an explanation for explanandum $\varphi$, and $\Sigma$ a $\varphi$-preserving selection function. The operator of explanation-based belief revision on $\B$ with $\fml{E}$ is defined as $\odot : 2^\mathcal{L} \times 2^\mathcal{L} \mapsto 2^\mathcal{L}$, $\B\odot \fml{E} = (\B\cup \fml{E}) \setminus \Sigma((\B\cup \fml{E})^\bot)$.
\end{definition}

The mechanism of the explanation-guided revision operator is to first add the explanation $\fml{E}$ to the belief base $\B$, and then retract from the result a correction set by means of a selection function that makes a choice among possible sets in the correction kernel of $\B\cup \fml{E}$, while also ensuring that the resulting belief base entails the explanandum. For the theoretically oriented reader, in the appendix we describe some basic postulates and an axiomatization of the explanation-guide operator.

\begin{example}
     Consider the belief base $\B$ and explanation $\fml{E}$ from Example \ref{example3}, and assume that we select $\Sigma ((\B\cup \fml{E})^\bot) = \{ \text{Wor}(Charlie) \rightarrow \text{Ins}(Charlie) \}$. Then the explanation-guided revision on $\B$ with $\fml{E}$ is $\B\odot \fml{E} = \{\text{Wor}(Charlie),  \text{Wor}(Charlie) \rightarrow \text{Ins}(Charlie), \lnot \text{Ins}(Charlie) \} \setminus \{ \text{Wor}(Charlie) \rightarrow \text{Ins}(Charlie) \} = \{\text{Wor}(Charlie), \lnot \text{Ins}(Charlie) \}$. Another possible revision is $\B\odot \fml{E} = \{\lnot \text{Ins}(Charlie) \}$. In fact, no matter the choice of the selection function, it is guaranteed that $\B\odot \fml{E} \models \lnot \text{Ins}$.
\end{example}

It is important to mention that the operator's result--specifically, the determination of which beliefs to retract--is influenced by the epistemic attitude of the agent, insofar as the agent's preexisting beliefs, level of skepticism, and openness to new information shape the outcome of belief revision. As we focus on the importance of explanations and non-minimal revisions in this paper, we leave this nuanced aspect of belief dynamics for future work.

\xhdr{A Simple Measure of Belief Change:}
Measures for calculating the amount of belief change typically depend on counting all the beliefs that change their values \cite{elio1997belief,harman1986change}. As such, to have an effective theoretical measure for quantifying the amount of change in a belief base, we propose the following definition:

\begin{definition}[Measure of Belief Change]
Let $\B$ be a prior belief base and $\B'$ a posterior belief base. We define the measure of belief change between $\B$ and $\B'$ as:
\[\fml{D}(\B, \B') = \frac{|\Gamma(\B) \vartriangle \Gamma(\B')|}{|\Gamma(\B) \cup \Gamma(\B')|},\] where $ \Gamma(\B) \vartriangle \Gamma(\B') = (\Gamma(\B) \setminus \Gamma(\B') \cup (\Gamma(\B') \setminus \Gamma(\B))$, and $\Gamma(\B)$ is the consequences of $\B$.
                                                                                                                    
\end{definition}

\begin{example}
\label{ex:change-measure}
    Consider the prior belief base $\B = \{\text{Wor}(Charlie)$, $\text{Wor}(Diana)$, $\text{Wor}(Charlie) \rightarrow \text{Ins}(Charlie)$, $\text{Wor}(Diana) \rightarrow \text{Ins}(Diana) \}$, with consequences $\Gamma(\B) = \{ \text{Wor}(Charlie), \text{Wor}(Diana), \text{Ins}(Charlie), \text{Ins}(Diana)\}$.
     
    Now, assume the explanation $\fml{E} = \{\text{Wor}(Charlie)$, $\text{Cop}(Charlie)$, $\text{Wor}(Charlie) \land \text{Cop}(Charlie)\rightarrow \lnot\text{Ins}(Charlie)\}$ for explanandum $\lnot \text{Ins}(Charlie)$, and consider the following revisions of $\B$ with $\fml{E}$:

    \squishlist
    \item \textbf{Minimal:} $\B' = \B \odot \fml{E} = \{$ $\text{Wor}(Charlie)$, $\text{Wor}(Diana)$, $\text{Wor}(Diana) \rightarrow \text{Ins}(Diana)$, $\text{Cop}(Charlie)$, $\text{Wor}(Charlie) \land \text{Cop}(Charlie)\rightarrow \lnot\text{Ins}(Charlie) \}$. The consequences of $\B'$ are $\Gamma(\B') = \{ \text{Wor}(Charlie), \text{Wor}(Diana), \text{Ins}(Diana)$, $\text{Cop}(Charlie)$, $\lnot \text{Ins}(Charlie)\}$. Computing the measure of belief change between $\B$ and $\B'$ we get $\fml{D}(\B, \B') = \frac{3}{6} = 0.5$.

   \item \textbf{Non-minimal:} $\B''= \B \odot \fml{E} =$ $\{\text{Wor}(Charlie)$, $\text{Wor}(Diana)$, $\text{Cop}(Charlie)$, $\text{Wor}(Charlie) \land \text{Cop}(Charlie)\rightarrow \lnot\text{Ins}(Charlie) \}$, with consequences  $\Gamma(\B'') = \{ \text{Wor}(Charlie), \text{Wor}(Diana)$, $\text{Cop}(Charlie)$, $\lnot \text{Ins}(Charlie)\}$. Then, computing the measure of belief change we get $\fml{D}(\B, \B'') = \frac{4}{6} = 0.66$.
    \squishend
\end{example}

As seen in the above example, a non-minimal revision obviously yields a higher belief change. At a first glance, non-minimal revisions might seem counter-intuitive--why should an agent discard more beliefs than what is minimally necessary for consistency? Are not these beliefs irrelevant to the explanandum anyway? But as we will see in the next section, these are not irrelevant beliefs. In fact, our operator allows us to capture how people revise their beliefs: \textit{non-minimally with explanations as their guide}.

\section{Human Belief Revision: Empirical Findings}
\label{sec:empirical}

\begin{quote}
    \textit{Do the dynamics of human belief revision align with the principle of minimalism or with explanatory understanding?} 
\end{quote}
This is the main question we are investigating in this section. As described earlier, the explanatory understanding suggests that in resolving inconsistencies, people seek explanations rather than simple minimal edits to their beliefs. The explanations then entail the revision, which may not be minimal. This is in contrast to the principle of minimalism, the most common principle in belief revision theory to date, which presupposes that an agent's primary goal when resolving inconsistencies is to make minimal changes in their beliefs.

To carry out our investigation, we conducted two human-subject study experiments using three types of inconsistent problems commonly referenced in cognitive science literature \citep{elio1997belief,politzer2001belief,byrne2002contradictions}. These problems were selected for their relevance in testing the depth of belief revision in response to inconsistencies. 

The first type (\textbf{Type I}) consists of a conditional generalization statement $S_1$, a (non-conditional) ground categorical statement $S_2$, and a fact $F$ that is inconsistent with what the statements imply. For example:

    \begin{itemize}
    \item $S_1$: \textit{If people are worried, then they find it difficult to concentrate.}
    \item $S_2$: \textit{Alice was worried.}
    \item $F$: \textit{In fact, Alice did not find it difficult to concentrate.}
    \end{itemize}

The second type (\textbf{Type II}) of inconsistency consists of two conditional generalization statements $S_1$ and $S_2$, a ground categorical statement $S_3$, and a fact $F$ that is inconsistent with the consequences of one of the conditional statements and the categorical statement. For example:
    \begin{itemize}
    \item $S_1$: \textit{If people are worried, then they find it difficult to concentrate.}
    \item $S_2$: \textit{If people are worried, then they have insomnia.}
    \item  $S_3$: \textit{Alice was worried.}
    \item $F$: \textit{In fact, Alice did not find it difficult to concentrate.}
    \end{itemize}

Finally, the third type (\textbf{Type III}) of inconsistency consists of two conditional generalizations $S_1$ and $S_2$, a categorical statement $S_3$, and a fact $F$ that is inconsistent with the consequences of both conditional statements and the categorical statement. For example:
    \begin{itemize}
    \item $S_1$: \textit{If people are worried, then they find it difficult to concentrate.}
    \item $S_2$: \textit{If people are worried, then they have insomnia.}
    \item  $S_3$: \textit{Alice was worried.}
    \item $F$: \textit{In fact, Alice did not find it difficult to concentrate and did not have insomnia.}
    \end{itemize}

In all three types, minimalism posits that a minimal resolution will be an explanation that rejects only the categorical statement (i.e., Alice was \textit{not} worried.). However, almost all conditional generalizations about events are susceptible to what psychologists refer to as \textit{disabling conditions} -- conditions describing how the conditional fails~\citep{dieussaert2000initial,elio1997belief,politzer2001belief}. For instance, ``\emph{Is it really the case that people find it difficult to concentrate when they are worried?}'' One can easily think of a disabling condition for this conditional, for example, ``\textit{people with effective coping strategies may still be able to concentrate despite being worried}''.

Because of people's propensity to envisage disabling conditions, their explanations are more likely to invoke such conditions than to imply that a categorical statement is wrong. But these explanations do not invoke a minimal change, because, logically speaking, they also remove the support for other consequences apart from the one giving rise to the inconsistency. For example, rejecting $S_1$ implies rejecting all of its groundings, which means you cannot infer that people find it difficult to concentrate if they are worried, for any instantiation of this rule (see also Example~\ref{ex:change-measure}) . Surely, this is not a minimal change.

The upcoming experiments aim to elucidate whether people's revisions follow minimalism or are more aligned with the explanatory understanding, providing some insights into human belief revision processes.

\begin{table*}[!ht]
\setlength{\tabcolsep}{4pt}
\centering
\small
\begin{tabular}{lccccc}
\toprule
\textbf{Problem Type} & \textbf{Total Valid Responses} & \textbf{Non-Minimal Explanation} & \textbf{Minimal Explanation} & \textbf{Wilcoxon test} & \textbf{Effect Size} \\
         & (Count)  &  (Count and \( \% \)) &  (Count and \( \% \)) & (\( p \)-value) &  (Cohen's \( d \))  \\
\midrule
\midrule
\textit{Type I} & 161 & 132 (81.99\%) & 29 (18.01\%) & \(4.76 \times 10^{-16}\) & 1.28  \\
\textit{Type II} & 161 & 140 (86.96\%) & 21 (13.04\%) & \(6.69 \times 10^{-21}\) & 1.48  \\
\textit{Type III} & 177 & 144 (81.36\%) & 33 (18.64\%) & \(7.23 \times 10^{-17}\) & 1.25  \\
\textit{Aggregate} & 499 & 416 (83.37\%) & 83 (16.63\%) & \(2.96 \times 10^{-50}\) & 1.33  \\
\bottomrule
\end{tabular}
\caption{Results from Experiment 1, with \textit{Aggregate} representing combined data from all problem types.}
\label{exp1}
\end{table*}

\begin{table*}[!t]
\setlength{\tabcolsep}{4pt}
\centering
\small
\begin{tabular}{lcccccc}
\toprule
        \textbf{Problem Type} &  \textbf{Total Valid Responses} &  \textbf{Non-Minimal Revision} & \textbf{\# of Changes}  &  \textbf{Minimal Revision} &  \textbf{Wilcoxon Test} &  \textbf{Effect Size} \\
         &  (Count)  &  (Count and \( \% \)) & (Avg.) &  (Count and \( \% \)) & (\( p \)-value) &  (Cohen's \( d \)) \\
\midrule
\midrule
         \textit{Type II} &              159 &          153 (96.23\%) & \(1.71\) &       6 (3.77\%) &          \(2.09 \times 10^{-31}\) & \(3.56\)  \\
        \textit{Type III} &              154 &          136 (88.31\%)& \(2.06\)   &       18 (11.69\%) &          \(1.93 \times 10^{-21}\) & \(1.89\) \\
      \textit{Aggregate} &              313 &          289 (92.33\%) & \(1.88\)  &       24 (7.67\%) &          \(1.01 \times 10^{-50}\) & \(2.43\)  \\
\bottomrule
\end{tabular}
\caption{Results from Experiment 2, with \textit{Aggregate} representing combined data from all problem types.}
\label{exp2}
\end{table*}

\subsection{Experiment 1}

Our first experiment looked at the three problem types described above and was aimed at providing some empirical data on what kinds of explanations do people seek in the face of inconsistencies. In other words, do people seek explanations that resolve categorical or conditional statements?

\xhdr{Participants and Design}
We recruited 62 participants from the online crowdsourcing platform Prolific \citep{palan2018prolific} across diverse demographics, with the only filter being that they are fluent in English. The participants carried out three different problems of each of the three types (Type I, Type II, and Type III), for a total of nine problems. The statements were taken from common, everyday events including subjects such as economics, intuitive physics, and psychology. The conditional statements in all problems were selected to be highly plausible and interpretable, similar to those in the high-plausibility category used by \citet{politzer2001belief}. All problem sets as well as more details of the study can be found in the appendix.

\noindent The participants' main task was to explain the inconsistencies presented to them, and we examined the revisions implied by their explanations. After providing their explanations for every problem, each participant was asked a question about how they approached explaining what was going on and if they followed any strategies when doing so. They also answered a Likert-type question about whether being provided an explanation will help them understand the inconsistency.

\xhdr{Results}
All participants came up easily with reasons to explain the inconsistencies they encountered. To analyze the results, we employed a coding scheme similar to that by \citet{byrne2002contradictions}. Explanations provided by the participants were categorized into two main types: (1)~Those implying non-minimal revisions (e.g., revisions to conditional generalizations), and (2)~those implying minimal revisions (e.g., revisions to categoricals). Explanations implying non-minimal revisions were either disabling conditions that would prevent the consequences of the generalization, or of the form ``It is not the case that if X then Y'', ``X is not sufficient for Y'', and other similar ones. Explanations that implied minimal revisions rejected the categorical statements and were of the form ``not X'', ``perhaps not X'', and so on.  This coding scheme classified $89\%$ of the responses. The remaining responses either affirmed or denied the new information, or were too vague to classify.

Table~\ref{exp1} illustrates the distribution of explanations implying either non-minimal or minimal revisions. The data reveal a compelling trend: an overwhelming majority of explanations across all questions leaned towards non-minimal revisions. A Wilcoxon test performed on the aggregated data yielded a $p$-value significantly smaller than 0.05 ($p \approx 2.96 \times 10^{-50}$), providing robust evidence that the observed proportions of non-minimal and minimal classifications are far from what would be expected by random chance. Specifically, non-minimal revisions were substantially more frequent than minimal revisions.

To probe the robustness of this finding, we conducted individual statistical tests for each question. Wilcoxon tests for each problem revealed $p$-values well below the $0.05$ threshold, affirming the prevalence of non-minimal revisions over minimal. Moreover, effect size measurements (Cohen's $d$) were conducted to quantify the magnitude of these differences, where it was consistently high across all instances. 

Collectively, these results offer empirical support for the prevalence of non-minimal revisions in participants' explanations. This inclination suggests that individuals engage deeply in resolving inconsistencies, often opting for more comprehensive explanatory frameworks that necessitate altering their existing beliefs to a greater extent than minimalism would predict. These results also provide insights into what kinds of explanations people tend to create.

\subsection{Experiment 2}

In this experiment we look at how people actually revise their beliefs when they are given an explanation for an inconsistency. 

\xhdr{Participants and Design}
We recruited $60$ participants from the Prolific platform with the same requirements as before. In this study, rather than having the participants generate their own explanations, they were presented with some of the most plausible explanations (that are disabling conditions) created by participants in Experiment 1, and then asked to describe how they would revise their information in light of the explanation. To ensure that they do not discard the explanations, we added some validity to the explanation by telling the participants that the explanation comes from a trustworthy source. Unlike in Experiment 1, however, the participants were only shown problems of Type II and III.

\xhdr{Results}
We employed a specific coding scheme to analyze how participants chose to revise their beliefs. In accordance with this scheme, participants indicated whether they would \textit{keep}, \textit{discard}, or \textit{alter} the beliefs.\footnote{We adopted a measure of belief change similar to those used in previous studies \cite{elio1997belief,harman1986change,walsh2009changing,khemlani2013cognitive}, which typically count the number of beliefs that change their values.}, and there does not exist another readily testable measure of belief change, we adopt a similar measure too. When choosing to alter a belief, participants were asked to provide details about how they would go about it. Like before, a minimal revision is one that discards or alters the categorical statement, and a non-minimal revision one that discards or alters either a generalization, or a combination of more than two statements. This coding scheme classified $87\%$ of the responses, while the remaining responses were either yielding inconsistent revisions (e.g., not revising anything) or too vague to classify.

Table~\ref{exp2} provides an overview of the results. The data reveal a clear trend: A significant majority of revisions were non-minimal across both problem types. In Type II problems, $96.23\%$ of the responses were non-minimal with an average of $1.65$ changes to beliefs, compared to $3.77\%$ that were minimal.  A Wilcoxon test produced a $p$-value of $2.09 \times 10^{-31}$, and the effect size $d$ was $3.56$. In Type III problems, $88.31\%$ were non-minimal and $11.69\%$ were minimal. The average number of changes were $1.82$. The Wilcoxon $p$-value was $1.93 \times 10^{-21}$, and the effect size $d$ was $2.06$. Moreover, aggregated data showed $92.33\%$ non-minimal revisions, with an average of $1.88$ belief changes, and $7.67\%$ minimal revisions. The Wilcoxon $p$-value was $1.01 \times 10^{-50}$ and the effect size $1.73$.

These findings not only strongly corroborate those of Experiment 1, but further solidify the evidence that people predominantly opt for non-minimal revisions when presented with explanations. Interestingly, even if we relax the assumption of what constitutes a minimal revision in our coding scheme, i.e., by supposing that it is discarding or altering one statement (conditional or generalization), these results clearly indicate that the average number of belief changes in people tend to be more than one. 

In summary, the results from both Experiment 1 and Experiment 2 offer compelling empirical data in the context of human belief revision processes. Far from making minimal changes to their existing beliefs, participants predominantly favored explanations that led to more comprehensive revisions. This demonstrates a natural inclination towards understanding the underlying factors that give rise to inconsistencies rather than merely resolving them in a superficial, yet minimal manner. The findings not only validate the explanatory understanding but also suggest a potential need for reevaluation of belief revision theories aimed at modeling actual humans. Our proposed human-aware belief revision framework is step towards this direction.

\section{Discussion and Conclusion}
\label{sec:discussion}

The field of belief revision theory has experienced remarkable progress, primarily influenced by the foundational work of Alchourrón, Makinson, and Gärdenfors. Their studies on revisions in legal codes \cite{alchourron1981hierarchies}, the introduction of rationality postulates for change operators \cite{gardenfors1982rules}, and the development of the \textit{AGM model} \cite{alchourron1985logic} have set the stage for subsequent advancements in the area. It is well known that one of the basic conceptual principles underlying the AGM model, as well as most belief revision frameworks, is the principle of minimalism. 

Contrary to the minimalist approach that has dominated belief revision theory, our empirical findings point towards a different paradigm--they suggest that in certain situations, individuals might opt for broader revisions to their belief systems, driven by the desire for a more comprehensive understanding and explanation of the information they encounter. In particular, our user studies revealed a notable propensity among participants to favor non-minimal, explanation-driven revisions when faced with inconsistencies. This preference persisted across different scenarios, suggesting that such an inclination might be a fundamental aspect of human reasoning. This finding aligns with the explanatory understanding, which proposes that humans prioritize generating coherent and plausible explanations over merely maintaining consistency with minimal changes.

We believe that these findings can have implications that extend to the burgeoning domains of explainable AI \citep{gunning2019xai} and human-aware AI \citep{kambhampati2020challenges}. As efforts in these fields converge towards fostering transparent, explainable, and synergistic interactions between humans and AI systems, aligning the AI systems' decision-making processes with human cognitive models can not only enhance explainability but also elevate the efficacy of human-AI collaborations.

A direct connection can be made with the Model Reconciliation Problem (MRP)~\citep{chakraborti2017plan,sreedharan2021foundations,son2021model,vas21,vasileioulogic}. Traditionally, MRP has been approached with an emphasis on minimal changes to the human user's (mental) model. However, our findings suggest that such human-aware AI systems might need to accommodate more substantial revisions to align more closely with human cognitive processes. Integrating our explanation-based belief revision framework into human-aware AI systems could lead to systems better equipped for effective interaction with humans, reflecting the complexity and depth of human belief revision. As AI systems progressively permeate society's decision-making structures, the imperative to align AI processes with human cognition becomes ever more pressing.

In conclusion, this paper contributes a new perspective to belief revision theory by introducing and empirically grounding a human-aware belief revision framework. Our empirical findings provide support for this framework, highlighting the importance of explanations in human belief revision processes, and the divergence from minimalism. Ultimately, belief revision is not an isolated process, but an integral component of people's broader quest for explanatory understanding.

\bibliography{aaai25}

\clearpage

\section*{{\LARGE Appendix}}

\section{Rationality Postulates and Axiomatization of Explanation-guided Revision Operator}
\label{postulates}

Following the tradition in belief revision literature \cite{gardenfors1982rules,falappa2002explanations,ferme2018belief}, we now present some basic rationality postulates for the explanation-guided belief revision. Rationality postulates are fundamental guidelines dictating the behavior of a revision operator--they specify what the operator's response should be, when provided with certain inputs, but not its internal mechanism.\footnote{The internal mechanism is specified in the construction of the operator we presented in the previous section.} 

We consider the following postulates:

\begin{itemize}

    \item \emph{Inclusion}: $\B \odot E \subseteq \B \cup E$.\\
    This postulate establishes that if the agent revises its belief base $\B$ with the explanation $E$, then the new belief base will stem from the union of $\B$ and $E$.

    \item \emph{Vacuity}: If $\B \cup E \not \models \bot$, then $\B \odot E = \B \cup E$.\\
    This postulate establishes that if the belief base $\B$ is consistent with the explanation $E$, then the revision simply expands $\B$ with $E$.

    \item \emph{Consistency}: If $\B \cup E \models \bot$, then $\B \odot E \not \models \bot$.\\
    This postulate ensures that the revised belief base is consistent.


    \item \textit{Reversion:} If $(\B \cup E)^\bot = (B \cup E')^\bot$, then $(B \cup E) \setminus (\B \odot E) = (B \cup E') \setminus (\B \odot E')$.

    This postulates states that if, if $\B \cup E$ and $\B \cup E'$ have the same correction kernels, then the beliefs retracted in the respective revisions with respect to $E$ and $E'$ are the same.

    \item \emph{Constrained Acceptance}: If $\B \not \models \neg \varphi$, then $\B \odot E \models \varphi$. \\
    This postulate establishes that if the agent receives an explanation $E$ for an explanandum $\varphi$ that is not rejected in its belief base $\B$, then $\varphi$ will be accepted in its revised belief base.

    \item \emph{Unconstrained Acceptance}: If $\B \models \neg \varphi$, then $\B \odot E \models \varphi$.\\
     This postulate establishes that if the agent receives an explanation $E$ for an explanandum $\varphi$ that is rejected in its belief base $\B$, then $\varphi$ will be accepted in its revised belief base.

    \item \emph{Strong Acceptance}: $\B \odot E \models \varphi$.\\
    This postulate ensures that the explanandum $\varphi$ can be inferred from the revised belief base.

\end{itemize}

Two obvious relations between the postulates follow.

\begin{proposition}
    If the operator $\odot$ satisfies vacuity, then it satisfies consistency and strong acceptance. 
\end{proposition}

\begin{proof}
    From vacuity, we have that $\B \cup E \not \models \bot$ and $\B \odot E = \B \cup E$. It follows then that $\B \odot E \not \models \bot$, satisfying consistency. Now, from the definition of an explanation we have that $E \models \varphi$, and from the monotonicity of classical logic we have that $\B \odot E \models \varphi$. Therefore, strong acceptance is satisfied. 
\end{proof}

\begin{proposition}
    If the operator $\odot$ satisfies strong acceptance, then it satisfies constrained acceptance and unconstrained acceptance.
\end{proposition}

\begin{proof}
    From strong acceptance, $\B \odot E \models \varphi$, thus automatically satisfying constrained and unconstrained acceptance.
\end{proof}

Finally, an axiomatic characterization of the explanation-guided belief revision operator is as follows.

\begin{theorem}
Let $\B$ be a belief base and $E$ an explanation for explanandum $\varphi$. The operator $\odot$ is an explanation-guided belief revision operator on $\B$ with $E$ if and only if it satisfies inclusion, consistency, reversion, and strong acceptance.
\end{theorem}

\begin{proof}
We first prove in the direction of construction to postulates. Let $\B\odot E = (\B\cup E) \setminus \Sigma((\B\cup E)^\bot)$:
\begin{itemize}
    \item  \textit{Inclusion}: Since $\B \odot E = (\B \cup E) \setminus \Sigma((\B \cup E)^\bot)$, it follows that $\B \odot E \subseteq \B \cup E$.
    
    \item  \textit{Consistency}: If $\B \cup E \models \bot$, then the operator retracts a correction set $\Sigma((\B \cup E)^\bot)$ from $\B \cup E$, thus rendering it consistent.

    \item \textit{Reversion}: Suppose that $(\B \cup E)^\bot = (B \cup E')^\bot$. Because $\Sigma$ is well-defined we have that $\Sigma((B_A \cup E)^\bot)= \Sigma((B_A \cup E')^\bot)$. Now, if $\alpha \in (\B \cup E)\setminus \B \odot E$, then $\alpha \in \Sigma((B_A \cup E)$, and consequently, $\alpha \in \B \cup E'$ and $\alpha \not \in \B \odot E'$. Hence, $(\B \cup E)\setminus \B \odot E \subseteq (\B \cup E')\setminus \B \odot E'$. Similarly, starting with $\alpha \in (\B \cup E')\setminus \B \odot E'$, we get that $(\B \cup E')\setminus \B \odot E' \subseteq (\B \cup E)\setminus \B \odot E$, thereby satisfying reversion.

    \item \textit{Strong Acceptance}: Since $E \models \varphi$, and since the $\varphi$-preserving selection function $\Sigma((\B\cup E)^\bot)$ selects sets that do not affect the entailment of $\varphi$ in $\B \cup E$, it follows that $(\B \cup E) \setminus \Sigma((\B\cup E)^\bot) \models \varphi$, thus satisfying strong acceptance. 
    
\end{itemize}
We now prove in the reverse direction from postulates to construction. 
We need show that if a belief revision operator satisfies the postulates then it is possible to build the explanation-guided belief revision operator. Let $\Sigma$ be a function such that for every pair of belief bases $\B$ and $E$, it holds that: $\Sigma((\B \cup E)^\bot)= \{\alpha \: : \: \alpha \in (\B \cup E) \setminus \B \odot E \}$. We must show that:
\begin{enumerate}
    \item $\Sigma$ is well-defined.\\
    That is, if $E$ and $E'$ are belief bases such that $(\B \cup E)^\bot = (\B \cup E')^\bot$, we will show that $\Sigma((\B \cup E)^\bot) = \Sigma((\B \cup E')^\bot)$. Since $E$ and $E'$ have the same correction kernels, it follows from reversion that $(\B \cup E) \setminus \B \odot E = (\B \cup E') \setminus \B \odot E'$. Therefore, $\Sigma((\B \cup E)^\bot) = \{\alpha \: : \: \alpha \in (\B \cup E) \setminus \B \odot E \} = \{\alpha \: : \: \alpha \in (\B \cup E') \setminus \B \odot E' \} = \Sigma((\B \cup E')^\bot)$. Hence, $\Sigma$ is well-defined.

    \item $\Sigma((\B \cup E)^\bot) \in (\B \cup E)^\bot$.\\
    Let $\alpha \in \Sigma((\B \cup E)^\bot)$. Then, $\alpha \in (\B \cup E) \setminus \B \odot E$. Since $\alpha \in (\B \cup E)$ and $\alpha \not \in \B \odot E$, there is a correction set $\B' \subseteq \B \cup E$ such that $\alpha \in \B'$. Therefore, $\alpha \in (\B \cup E)^\bot$.

    \item $(\B \cup E) \setminus \Sigma((\B \cup E)^\bot) \models \varphi$.
    Follows from the strong acceptance postulate $\B \odot E \models \varphi$.
\end{enumerate}

From the inclusion postulate and the definition of $\Sigma((\B \cup E)^\bot)$, we can conclude that the operator is an explanation-guided belief revision operator.
\end{proof}

\section{Illustrative Example with Existing Belief Revision Operator}

A reader familiar with the belief revision literature may wonder what is the difference between the explanation-guided belief revision operator we presented in this section and existing belief base revision frameworks. First, we want to note that our framework is specifically intended for human-aware AI settings. Now, closest to our operator is the belief revision operator proposed by \citet{falappa2002explanations}, which considers the revision of belief bases with explanations (or sets of formulae). The following example is aimed at illustrating the primary differences between the two approaches.

The belief revision operator of \citet{falappa2002explanations}, referred to by the authors as kernel revision by a set of sentences, is defined as follows:

\begin{definition}[\citet{falappa2002explanations} Revision Operator]
Let $\B$ be a belief base and $E$ an explanation. The kernel revision by a set of sentences $\diamond_F$ is defined as $\B \diamond_F E = (\B \cup E) \setminus \sigma((\B \cup E)^{\ind} \bot) $, where $(\B \cup E)^{\ind} \bot$ is a kernel set \cite{hansson1994kernel} (e.g., minimal unsatisfiable subsets of $\B \cup E$), and $\sigma((\B \cup E)^{\ind} \bot)$ an incision function such that $\sigma((\B \cup E)^{\ind} \bot) \subseteq \bigcup (\B \cup E)^{\ind} $ and if $X\in(\B \cup E)^{\ind}$ and $X \not = \emptyset$ then $X \cap \sigma((\B \cup E)^{\ind} \bot) = \emptyset$.
\end{definition}

First, we see from the above definition that the revision by \citet{falappa2002explanations} does not satisfy explanatory understanding, i.e., it makes it possible for the revised belief base to not entail the explanandum. Second, it follows the principle of by minimalism. Let us illustrate the differences in the following example.

Let $w_c$ be $\text{Wor}(Charlie)$, $w_d$ be $\text{Wor}(Diana)$, $c_c$ be $\text{Cop}(Charlie)$, $i_c$ be $\text{Ins}(Charlie)$, and $i_d$ be $\text{Ins}(Diana)$. 
    
Consider the belief base: 
$$\B = \{ w_c, w_d, w_c \rightarrow i_c, w_d \rightarrow i_d\},$$ 

which has consequences $\Gamma(\B) = \{ w_c, w_d, i_c, i_d \}$. Now, assume the following explanation for explanandum $\lnot i_c$:

$$\fml{E} = \{  w_d, c_c, w_c \land c_c \rightarrow \lnot i_c \}. $$

Using the framework by \citet{falappa2002explanations} to revise $\B$ with $\fml{E}$, we go through the following steps: 
  \begin{enumerate}
      \item \textbf{Kernel Set:} $(\B \cup \
\fml{E})^{\ind} \bot =\{ w_c, c_c, w_c \rightarrow i_c, w_c \land c_c \rightarrow \lnot i_c  \}\}$.

\item \textbf{Incision Function:} Possible results of $\sigma((\B \cup E)^{\ind} \bot)$ are: $\{w_c\}$, $\{  w_c \rightarrow i_c\}$ $\{ w_c \land c_c \rightarrow \lnot i_c\}$, $\{c_c, w_c \land c_c \rightarrow \lnot i_c\}$, and so on.
  \end{enumerate}  
    
What is important to highlight here are two things: (1) \textit{The incision function only selects subsets to retract from the minimal unsatisfable set, i.e., only those beliefs deemed ``relevant'' wrt minimality}; and (2) \textit{It is possible for the revised belief base to reject the explanandum}, i.e., $\B \diamond_F \fml{E} \not \models \lnot i_c$. 

In contrast, our operator works differently. First, it ensures that the revised belief base always entails the explanandum, i.e., $\B \odot E \models \lnot i_c$, in all possible results of the operator, while also not constrained to choose beliefs based on a minimality criterion. Importantly, notice how $w_d \rightarrow i_d \in \B$ will never be in any of the possible solutions of Falappa et al.'s operator, as it is not included in any of the kernels in $(\B \cup E)^{\ind} \bot$. However, in our case $w_d \rightarrow i_d \in (\B \cup E)^\bot$, and as such, it is possible for it to be in the solution of our operator. Again, while this might seem as an ``irrelevant'' revision that should be avoided (i.e., violates the minimal change principle), our empirical findings strongly support it, namely that given an explanation of an inconsistency, people, more frequently than not, make non-minimal revisions to their beliefs.

Finally, it is very important to note that we do not wish to replace any of the existing belief revision operators from the literature. Each has its own suitability and application. What we wish to do is present a framework that is tailored to human belief revision processes. That is, our operator is intended for human-aware AI settings.

\section{Human-Subject Experiments}

\subsection{Experiment 1}

Our first experiment looked at the three problem types described above and was aimed at providing some empirical data on whether people seek explanations that tend to invoke non-minimal changes to their beliefs, that is, do people seek explanations that resolve categorical or conditional propositions? The participants main task was to explain the inconsistencies presented to them, and we examined the revisions implied by their explanations.

\subsubsection{\textbf{Participants and Design}}

We recruited 62 participants from the online crowdsourcing platform Prolific \citep{palan2018prolific} across diverse demographics, with the only filter being that they are fluent in English. The participants carried out three different problems of each of three types (Type I, Type II, and Type III). The propositions were taken from common, everyday events including subjects such as economics and psychology. The conditional propositions in all problems were selected to be highly plausible and interpretable, similar to those in the high-plausibility category used by \citet{politzer2001belief}.

Each participant was given the following instructions: 
\begin{quote}
\textit{You will be presented with a series of everyday common scenarios. In each scenario, you will be presented with information from two or three different speakers talking about some specific things. You will then be given some additional information that you know, for a fact, to be true. Your task, in essence, is to explain what is going on.}
\end{quote}

We gave the following instructions to all participants:

\begin{itemize}
    \item \textit{\textbf{Read carefully:} For each scenario, read all the information very carefully.}
    \item \textit{\textbf{Explain}: Think about how to explain the fact. In other words, ask yourself: why does the fact conflict with the information provided by the speakers? Answer in your own words.}
    \item \textit{\textbf{No Right or Wrong Answers}: This study aims to understand your personal thought process. There are no right or wrong answers. Choose what feels most accurate to you.}
    \item \textit{\textbf{Pace Yourself}: While there's no strict time limit, try to spend a reasonable amount of time on each scenario—neither rushing through nor overthinking too much.}

\end{itemize}

\begin{figure*}[!ht]
    \centering
    \includegraphics[width=0.9\textwidth]{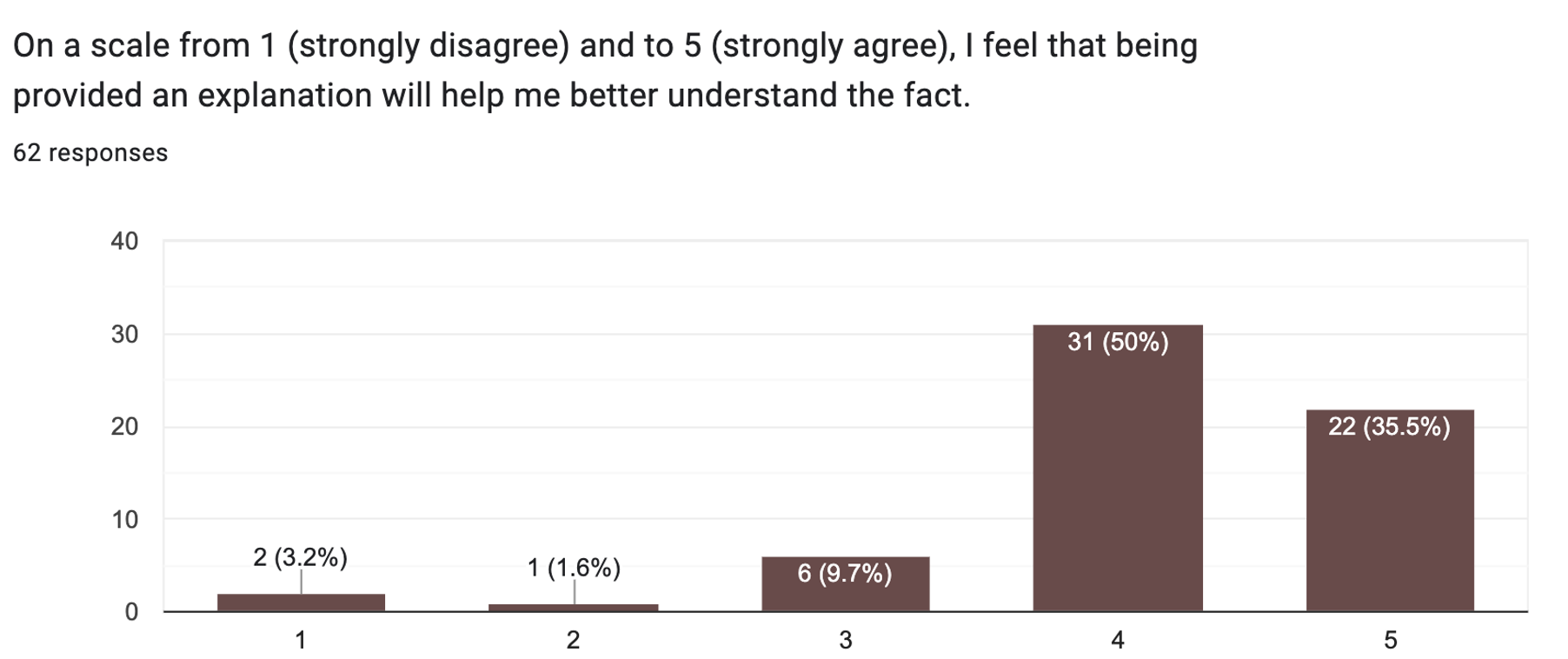}
    \caption{Distribution of Responses to Likert-type Question.}
    \label{fig:enter-label}
\end{figure*}

Afterwards, the participants saw the following scenarios and questions:\\

\noindent \textbf{Scenario 1:}
\begin{itemize}
    \item \textbf{$S_1$}: \textit{If a drink contains sugar, then it gives you energy.}
    \item \textbf{$S_2$}: \textit{This drink contains sugar.}
    \item \textbf{Fact}: \textit{In fact, it doesn't give you energy.}
\end{itemize}

\textit{Why does the drink not give you energy?}

\medskip \noindent \textbf{Scenario 2:}
\begin{itemize}
    \item \textbf{$S_1$}: \textit{If sales go up, then profits improve.}
    \item \textbf{$S_2$}: \textit{The sales went up.}
    \item \textbf{Fact}: \textit{In fact, the profits did not go up.}
\end{itemize}

\textit{Why did the sales not go up?}

\medskip \noindent \textbf{Scenario 3:}
\begin{itemize}
    \item \textbf{$S_1$}: \textit{If people have a fever, then they have a high temperature.}
    \item \textbf{$S_2$}: \textit{Maria had a fever.}
    \item \textbf{Fact}: \textit{In fact, Maria did not have a high temperature.}
\end{itemize}

\textit{Why did Maria not have a high temperature?}

\medskip \noindent \textbf{Scenario 4:}
\begin{itemize}
    \item \textbf{$S_1$}: \textit{If there is very loud music, then it is difficult to have a conversation.}
    \item \textbf{$S_2$}: \textit{If there is very loud music, then the neighbors complain.}
    \item  \textbf{$S_3$}: \textit{The music was loud.}
    \item \textbf{Fact}: \textit{In fact, the neighbors did not complain.}
\end{itemize}

\textit{Why did the neighbors not complain?}

\medskip \noindent \textbf{Scenario 5:}
\begin{itemize}
    \item \textbf{$S_1$}: \textit{If people are worried, then they find it difficult to concentrate.}
    \item \textbf{$S_2$}: \textit{If people are worried, then they have insomnia.}
    \item  \textbf{$S_3$}: \textit{Alice was worried.}
    \item \textbf{Fact}: \textit{In fact, Alice did not find it difficult to concentrate.}
\end{itemize}

\textit{Why did Alice not find it difficult to concentrate?}

\medskip \noindent \textbf{Scenario 6:}
\begin{itemize}
    \item \textbf{$S_1$}: \textit{If you follow this diet, then you lose weight.}
    \item \textbf{$S_2$}: \textit{If you follow this diet, then you have a good supply of iron}
    \item  \textbf{$S_3$}: \textit{John followed this diet.}
    \item \textbf{Fact}: \textit{In fact, John did not lose weight.}
\end{itemize}

\textit{Why did John not lose weight?}

\medskip \noindent \textbf{Scenario 7:}
\begin{itemize}
    \item \textbf{$S_1$}: \textit{If someone is very kind to you, then you like that person.}
    \item \textbf{$S_2$}: \textit{If someone is very kind to you, then you are kind in return.}
    \item  \textbf{$S_3$}: \textit{Jocko is very kind to Kristen.}
    \item \textbf{Fact}: \textit{In fact, Kristen did not like Jocko, and she were not kind in return.}
\end{itemize}

\textit{Why did Kristen not like Jocko and was not kind to him?}

\medskip \noindent \textbf{Scenario 8:}
\begin{itemize}
    \item \textbf{$S_1$}: \textit{If a match is struck, then it produces light.}
    \item \textbf{$S_2$}: \textit{If a match is struck, then it gives off smoke.}
    \item  \textbf{$S_3$}: \textit{Mary struck a match.}
    \item \textbf{Fact}: \textit{In fact, the match produced no light, and it did not give off smoke.}
\end{itemize}

\textit{Why did the match produce no light and gave off no smoke?}

\medskip \noindent \textbf{Scenario 9:}
\begin{itemize}
    \item \textbf{$S_1$}: \textit{If people are nervous, then their hands shake.}
    \item \textbf{$S_2$}: \textit{ If people are nervous, then they get butterflies in their stomach.}
    \item  \textbf{$S_3$}: \textit{Patrick was nervous.}
    \item \textbf{Fact}: \textit{In fact, Patrick's hands did not shake, and he didn't get butterflies in his stomach.}
\end{itemize}

\textit{Why did Patrick's hands not shake and he didn't get butterflies in his stomach?}

\medskip
After going through all nine scenarios, the participants were asked the following two questions:

\smallskip \noindent \textbf{Q1:} \textit{Describe in your own words how you approached explaining what was going on. Was there a specific reason why you chose to retain or discard certain information?} 

\smallskip \noindent \textbf{Q2:} \textit{On a scale from 1 (strongly disagree) and to 5 (strongly agree), I feel that being provided an explanation will help me better understand the fact.}

\smallskip \noindent Figure~1 shows the distribution of the Likert question (Q2).

\subsubsection{Experiment 2}

Building on the findings of Experiment 1, which demonstrated a strong tendency among participants to resolve inconsistencies through non-minimal revisions, in this experiment, we look at how people actually revise their beliefs when they are given an explanation. This experiment is also relevant to human-aware AI systems, where AI agents provide explanations to human users.

\subsection{\textbf{Participants and Design}}
We recruited $60$ participants from the Prolific platform with the same requirements as before. In this follow-up study, rather than having the participants generate their own explanations, they were presented with some of the most plausible explanations created by participants in Experiment 1, and then asked to describe how they would revise their information in light of the given explanation. To ensure that they will not discard the explanation, they were informed that the explanation is trustworthy. Unlike in Experiment 1, however, the participants were only shown problems of Type II and III. The reason is that these problem types contain more information (e.g., two conditionals and one categorical proposition), and thus it is easier to measure if their revisions are minimal or not.

\begin{figure*}[!t]
    \centering
    \includegraphics[width=0.9\textwidth]{response_histograms.pdf}
    \caption{Distribution of Belief Changes per Question.}
    \label{fig:changes}
\end{figure*}

The scenarios the participants saw can be seen below:

\medskip \noindent \textbf{Scenario 1:}
\begin{itemize}
    \item \textbf{$S_1$}: \textit{If there is very loud music, then it is difficult to have a conversation.}
    \item \textbf{$S_2$}: \textit{If there is very loud music, then the neighbors complain.}
    \item  \textbf{$S_3$}: \textit{The music was loud.}
    \item \textbf{Fact}: \textit{In fact, the neighbors did not complain.}
    \item \textbf{Explanation:} \textit{Explanation: If the neighbors are away on vacations, then very loud music does not lead to complaints.}
\end{itemize}

\medskip \noindent \textbf{Scenario 2:}
\begin{itemize}
    \item \textbf{$S_1$}: \textit{If people are worried, then they find it difficult to concentrate.}
    \item \textbf{$S_2$}: \textit{If people are worried, then they have insomnia.}
    \item  \textbf{$S_3$}: \textit{Alice was worried.}
    \item \textbf{Fact}: \textit{In fact, Alice did not find it difficult to concentrate.}

    \item  \textbf{Explanation:} \textit{If people have effective coping strategies, then they may still be able to concentrate despite being worried.}
\end{itemize}

\medskip \noindent \textbf{Scenario 3:}
\begin{itemize}
    \item \textbf{$S_1$}: \textit{If you follow this diet, then you lose weight.}
    \item \textbf{$S_2$}: \textit{If you follow this diet, then you have a good supply of iron}
    \item  \textbf{$S_3$}: \textit{John followed this diet.}
    \item \textbf{Fact}: \textit{In fact, John did not lose weight.}

    \item  \textbf{Explanation:} \textit{If people have metabolic imbalances, then following a particular diet may not result in weight loss.}
\end{itemize}

\medskip \noindent \textbf{Scenario 4:}
\begin{itemize}
    \item \textbf{$S_1$}: \textit{If someone is very kind to you, then you like that person.}
    \item \textbf{$S_2$}: \textit{If someone is very kind to you, then you are kind in return.}
    \item  \textbf{$S_3$}: \textit{Jocko is very kind to Kristen.}
    \item \textbf{Fact}: \textit{In fact, Kristen did not like Jocko, and she were not kind in return.}

    \item \textbf{Explanation:} \textit{If people have had negative past experiences with someone, then they may not like that person or reciprocate kindness despite the person being kind to them.}
    
\end{itemize}

\medskip \noindent \textbf{Scenario 5:}
\begin{itemize}
    \item \textbf{$S_1$}: \textit{If a match is struck, then it produces light.}
    \item \textbf{$S_2$}: \textit{If a match is struck, then it gives off smoke.}
    \item  \textbf{$S_3$}: \textit{Mary struck a match.}
    \item \textbf{Fact}: \textit{In fact, the match produced no light, and it did not give off smoke.}

    \item \textbf{Explanation:} \textit{If the match is wet, then it will neither produce light nor give off smoke.}
\end{itemize}

\medskip \noindent \textbf{Scenario 6:}
\begin{itemize}
    \item \textbf{$S_1$}: \textit{If people are nervous, then their hands shake.}
    \item \textbf{$S_2$}: \textit{ If people are nervous, then they get butterflies in their stomach.}
    \item  \textbf{$S_3$}: \textit{Patrick was nervous.}
    \item \textbf{Fact}: \textit{In fact, Patrick's hands did not shake, and he didn't get butterflies in his stomach.}

    \item \textbf{Explanation:} \textit{If individuals have practiced stress-management techniques, then they may not exhibit shaky hands or butterflies in the stomach when nervous.}
\end{itemize}

After each single scenario, the participants answered the following question:

\smallskip \noindent \textit{Describe in your own words how you will revise the information. Was there a specific reason why you chose to retain or discard information from the speakers? To be brief, you can write: keep \textbf{$S_1$}, discard \textbf{$S_1$}, alter \textbf{$S_1$}, and so on (if you alter, please describe how).}

\smallskip 
Finally, Figure \ref{fig:changes} shows the distribution of user belief changes per question asked.

\end{document}